\documentclass[12pt,reqno]{amsart}
\usepackage[left=1in,right=1in,top=1in,bottom=1in]{geometry}
\usepackage[foot]{amsaddr}
\usepackage{hyperref}
\usepackage{url}
\usepackage{natbib}
\usepackage{times}
\usepackage{amsfonts}
\usepackage{amsmath}
\usepackage[psamsfonts]{amssymb}
\usepackage{amstext}
\usepackage{amsthm}
\usepackage{latexsym}
\usepackage{color}
\usepackage{graphicx}
\usepackage{enumerate}
\usepackage{algorithm}
\usepackage{algorithmic}
\usepackage[mathscr]{euscript}
\usepackage{multirow}

\newtheorem*{rep@theorem}{\rep@title}
\newcommand{\newreptheorem}[2]{%
\newenvironment{rep#1}[1]{%
 \def\rep@title{#2 \ref{##1}}%
 \begin{rep@theorem}}%
 {\end{rep@theorem}}}

\makeatother

\hypersetup{
  colorlinks   = true,
  urlcolor     = blue,
  linkcolor    = blue,
  citecolor    = blue
}

\def\Rset{\mathbb{R}}

\DeclareMathOperator*{\E}{\rm E}
\DeclareMathOperator*{\Var}{\rm Var}

\DeclareMathOperator*{\argmin}{\rm argmin}
\DeclareMathOperator{\conv}{conv}
\DeclareMathOperator{\sgn}{sgn}
\DeclareMathOperator{\Tr}{Tr}

\DeclareMathOperator{\Pdim}{Pdim}
\DeclareMathOperator{\VCdim}{VC-dim}

\newcommand{\e}{\epsilon}
\newcommand{\h}{\widehat}

\newcommand{\R}{\mathfrak{R}}

\newcommand{\cF}{{\mathcal F}}

\newcommand{\cD}{{\mathcal D}}

\newcommand{\cN}{{\mathcal N}}
\newcommand{\cX}{{\mathcal X}}

\newcommand{\cG}{{\mathcal G}}

\newcommand{\scrH}{{\mathscr H}}
\newcommand{\Alpha}{{\boldsymbol \alpha}}
\newcommand{\ssigma}{{\boldsymbol \sigma}}
\newcommand{\mat}[1]{{\mathbf #1}}
\newcommand{\be}{\mat{e}}
\newcommand{\bh}{\mat{h}}

\newcommand{\bw}{\mat{w}}

\newcommand{\n}{\mat{n}}
\newcommand{\K}{\mat{K}}
\newcommand{\N}{\mat{N}}
\newcommand{\0}{\mat{0}}

\newcommand{\tts}{\tt \small}

\newcommand{\set}[1]{\{#1\}}

\newtheorem{theorem}{Theorem}
\newreptheorem{theorem}{Theorem}
\newtheorem{lemma}[theorem]{Lemma}
\newreptheorem{lemma}{Lemma}

\newreptheorem{corollary}{Corollary}

\newreptheorem{proposition}{Proposition}
\newtheorem*{theorem*}{Theorem}

\newcommand{\ignore}[1]{}

\newcommand{\Deep}{Voted}

\newcommand{\DeepSVM}{VKR}

\newcommand{\AlgoName}{Voted Kernel Regularization}

\title{\AlgoName}

\author{Corinna Cortes}
\address{Google Research,
111 8th Avenue, New York, NY 10011}
\email{corinna@google.com}

\author{Prasoon Goyal}
\address{Courant Institute of Mathematical Sciences,
251 Mercer Street, 
New York, NY 10012}
\email{pg1338@nyu.edu}

\author{Vitaly Kuznetsov}
\address{Courant Institute of Mathematical Sciences,
251 Mercer Street, 
New York, NY 10012}
\email{vitaly@cims.nyu.edu}

\author{Mehryar Mohri}
\address{Courant Institute and Google Research,
251 Mercer Street, 
New York, NY 10012}
\email{mohri@cs.nyu.edu}

\begin{document}

\maketitle

\begin{abstract}
  This paper presents an algorithm, \emph{\AlgoName~}, that provides
  the flexibility of using potentially very complex \ignore{(or \emph{deep})}
  kernel functions such as predictors based on much higher-degree
  polynomial kernels, while benefitting from strong learning
  guarantees.  The success of our algorithm arises from derived bounds
  that suggest a new regularization penalty in terms of the Rademacher
  complexities of the corresponding families of kernel maps. In a
  series of experiments we demonstrate the improved performance of our
  algorithm as compared to baselines. Furthermore,
  the algorithm enjoys several favorable properties. The optimization
  problem is convex, it allows for learning with non-PDS kernels, and
  the solutions are highly sparse, resulting in improved
  classification speed and memory requirements.
\end{abstract}

\section{Introduction}
\label{sec:intro}

The hypothesis returned by learning algorithms such as SVMs
\citep{CortesVapnik1995} and other algorithms for which the representer
theorem holds is a linear combination of functions $K(x, \cdot)$,
where $K$ is the kernel function used and $x$ is a training sample. The
generalization guarantees for SVMs depend on the sample size and the
margin, but also on the complexity of the kernel function $K$ used,
measured by its trace \citep{KoltchinskiiPanchenko2002}.

These guarantees suggest that, for a moderate margin, learning with
very complex kernels, such as sums of polynomial kernels of degree up
to some large $d$ may lead to overfitting, which frequently is observed
empirically. Thus, in practice, simpler kernels are typically used,
that is small $d$s for sums of polynomial kernels.  On the other hand,
to achieve a sufficiently high performance in challenging learning
tasks, it may be necessary to augment a linear combination of such
functions $K(x, \cdot)$ with a function $K'(x, \cdot)$, where $K'$ is
possibly a substantially more complex kernel, such as a polynomial
kernel of degree $d' \gg d$. This flexibility is not available when
using SVMs or other learning algorithms such as kernel Perceptron \citep{AizermanBravermanRozonoer64,Rosenblatt} with
the same solution form: either a complex kernel function $K'$ is used
and then there is risk of overfitting, or a potentially too simple
kernel $K$ is used \ignore{disallowing access to more complex $K'(x, \cdot)$ and}
limiting the performance that could be achieved in some tasks.

This paper presents an algorithm, \emph{\AlgoName~}, that precisely
provides the flexibility of using potentially very complex \ignore{(or
\emph{deep})} kernel functions such as predictors based on much
higher-degree polynomial kernels, while benefitting from strong
learning guarantees.  In a series of experiments we demonstrate the
improved performance of our algorithm.

We present data-dependent learning bounds for
this algorithm that are expressed in terms of the Rademacher
complexities of the reproducing kernel Hilbert spaces (RHKS) of the
kernel functions used.
These results are based on the framework of \emph{Voted Risk Minimization}
originally introduced by \cite{CortesMohriSyed2014} for
ensemble methods.
We further extend these results
using local Rademacher complexity analysis to
show that faster convergence rates are possible when
the spectrum of the kernel matrix is controlled. The success of our algorithm arises from these bounds that suggest a
new regularization penalty in terms of the Rademacher complexities of
the corresponding families of kernel maps. Therefore, it becomes
crucial to have a good estimate of these complexity measures. We
provide a thorough theoretical analysis of these complexities for
several commonly used kernel classes.
 
Besides the improved performance and the theoretical guarantees
\AlgoName~ admits a number of additional favorable properties. Our
formulation leads to a convex optimization problem that can be solved
either via Linear Programming or using Coordinate Descent.  \AlgoName~
does not require the kernel functions to be positive-definite or even
symmetric. This enables the use of much richer families of kernel
functions.  In particular, some standard distances known not to be PSD
such as the edit-distance and many others can be used with this
algorithm.

Yet another advantage of our algorithm is that it produces highly
sparse solutions providing greater efficiency and less memory
needs. In that respect, \AlgoName~ is similar to so-called \emph{norm-1
  SVM} \citep{Vapnik98,ZhuRossetHastieTibshirani2003} and
\emph{Any-Norm-SVM} \citep{dekel2007support} which all use a
norm-penalty to reduce the number of support vectors. However, to the
best of our knowledge these regularization terms on their own has not
led to performance improvement over regular SVMs
\citep{ZhuRossetHastieTibshirani2003,dekel2007support}.  In contrast,
our experimental results show that \AlgoName~ algorithm can outperform
both regular SVM and norm-1 SVM, and at the same time significantly
reduce the number of support vectors. In other work hybrid
regularization schemes are combined to obtain a performance
improvement \citep{zou2007improved}. Possibly this technique could be
applied to our \AlgoName~ algorithm as well resulting in additional
performance improvements.

Somewhat related algorithms are learning kernels or
multiple kernel learning and has been extensively investigated
over the last decade by both algorithmic and theoretical studies
\citep{lanckriet,argyriou-colt,argyriou-icml,shai,lewis-et-al,zienO07,
  micchelli-and-pontil,jebara04,bach,\ignore{l2reg,}ong,ying,lk}. In
learning kernels, training data is used to select a single kernel out
of the family of convex combinations of $p$ base kernels and to learn
a predictor based on just one kernel. In contrast in \Deep~ SVM, every
training point can be thought of as representing a different
kernel. Another related approach is \emph{Ensemble SVM}
\citep{ensembleSVM}, where a predictor for each base kernel is used
and these predictors are combined in to define a single predictor,
these two tasks being performed in a single stage or in two subsequent
stages. The algorithm where the task is performed in a single stage
bears the most resemblance with our \AlgoName~. However the
regularization is different and most importantly not capacity-dependent.

The rest of the paper is organized as follows. Some preliminary definitions and notation are introduced in Section~\ref{sec:prelims}. The \AlgoName~
algorithm is presented in Section~\ref{sec:algo}
and in Section~\ref{sec:guarantees} we provide strong data-dependent
learning guarantees for this algorithm showing
that it is possible to learn with highly complex \ignore{or deep }kernel
classes and yet not overfit. In Section~\ref{sec:guarantees},
we also prove local complexity bounds that detail
how faster convergence rates are possible provided that
the spectrum of the kernel matrix is controlled.
Section~\ref{sec:opt} discusses the 
implementation of the \AlgoName~ algorithm including optimization procedures
and analysis of Rademacher complexities. We conclude with
experimental results in Section~\ref{sec:experiments}.

\section{Preliminaries}
\label{sec:prelims}

Let $\cX$ denote the input space. We consider the familiar supervised
learning scenario. We assume that training and test points are drawn
i.i.d.\ according to some distribution $\cD$ over
$\cX \times \set{-1, +1}$ and denote by
$S = ((x_1, y_1), \ldots, (x_m, y_m))$ a training sample of size $m$
drawn according to $\cD^m$.

Let $\rho > 0$. For a function $f$ taking values in $\Rset$, we denote
by $R(f)$ its binary classification error, by $\h R_S(f)$ its
empirical error, and by $\h R_{S, \rho}(f)$ its empirical margin
error for the sample $S$:
\begin{align*}
 R(f) = \E_{(x, y) \sim \cD} [1_{y f(x) \leq 0}],
\quad \h R_S(f) = \E_{(x, y) \sim S} [1_{y f(x) \leq 0}],
\text{ and } \h R_\rho(f) = \E_{(x, y) \sim S} [1_{y f(x) \leq \rho}],
\end{align*}
where the notation $(x, y) \sim S$ indicates that $(x, y)$ is
drawn according to the empirical distribution defined by $S$. We will denote by $\h \R_S(H)$ the empirical Rademacher complexity of
a hypothesis set $H$ on the set $S$ of functions mapping $\cX$ to $\Rset$, and by
$\R_m(H)$ the Rademacher complexity
\citep{KoltchinskiiPanchenko2002,BartlettMendelson2002}:
\begin{equation*}
\h \R_S(H) = \frac{1}{m} \E_{\ssigma} \bigg[\sup_{h \in H} \sum_{i = 1}^m \sigma_i
h(x_i) \bigg] \qquad \R_m(H) = \E_{S \sim D^m} \Big[\h \R_S(H) \Big],
\end{equation*}
where the random variables $\sigma_i$ are independent and uniformly 
distributed over $\set{-1, +1}$.

\ignore{
\subsection{Local Rademacher Complexity Bounds}
}

\section{The \AlgoName~ Algorithm}
\label{sec:algo}

In this section, we introduce the \AlgoName~ algorithm.  Let
$K_1, \ldots, K_p$ be $p$ positive semi-definite (PSD) kernel
functions with $\kappa_k = \sup_{x \in \cX} \sqrt{K_k(x, x)}$ for all
$k \in [1, p]$.  We consider $p$ corresponding
families of functions mapping from $\cX$ to $\Rset$,
$H_1, \ldots, H_p$, defined by $
H_k = \set{x \mapsto \pm K_k(x, x') \colon x' \in \cX} $,
where the sign accounts for two possible ways of classifying a point
$x' \in \cX$.
The general form of a hypothesis $f$ returned by the algorithm is
the following:
\begin{equation*}
\label{eq:f}
f = \sum_{j = 1}^m \sum_{k = 1}^p \alpha_{k, j} K_k(\cdot, x_j),
\end{equation*}
where $\alpha_{k, j} \in \Rset$ for all $j$ and $k$. Thus, $f$ is a
linear combination of hypotheses in $H_k$s. This form with many
$\alpha$s per point is distinctly different from that of learning
kernels with only one $\alpha$ per point. Since the
families $H_k$ are symmetric, this linear combination can be made a
non-negative combination. Our algorithm consists of minimizing the
Hinge loss on the training sample, as with SVMs, but with a different
regularization term that tends to penalize hypotheses drawn from more
complex $H_k$s more than those selected from simpler ones and to
minimize the norm-1 of the coefficients $\alpha_{k, j}$. Let $r_k$
denote the empirical Rademacher complexity of $H_k$:
$r_k = \h \R_S(H_k)$. Then, the following is the objective function
of \AlgoName~:
\begin{align}
\label{eq:dsvm_objective}
 \! F(\Alpha) \!=\! \frac{1}{m} \!\sum_{i = 1}^m \max\bigg(0, 1
    \!-\! y_i y_j
 \sum_{j = 1}^m \sum_{k=1}^p \alpha_{k,j} K_k(x_i, x_j) \bigg) \!
+\! \sum_{j =1}^m \sum_{k=1}^p (\lambda r_k + \beta)
  |\alpha_{k, j}|, \!
\end{align}
where $\lambda \geq 0$ and $\beta \geq 0$ are parameters of the
algorithm. We will adopt the notation
$\Lambda_k = \lambda r_k + \beta$ to simplify the presentation in what
follows.

Note that the objective function $F$ is convex: the Hinge loss is
convex thus its composition with an affine function is also convex,
which shows that the first term is convex; the second term is convex
as the absolute value terms with non-negative coefficients; and $F$ is
convex as the sum of these two convex terms.  Thus, the optimization
problem admits a global minimum.  \AlgoName~ returns the function $f$
defined by \eqref{eq:f} with coefficients
$\Alpha = (\alpha_{k, j})_{k, j}$ minimizing $F$.

This formulation admits several benefits.
First, it enables us to learn with very complex hypothesis sets and
yet not overfit, thanks to Rademacher complexity-based penalties
assigned to coefficients associated to different $H_k$s. We will see
later that the algorithm thereby defined benefits from strong learning
guarantees. Notice further that the penalties assigned are
data-dependent, which is a key feature of the algorithm.  Second,
observe that the objective function \eqref{eq:deepboost_objective}
does not require the kernels $K_k$ to be positive-definite or even
symmetric. Function $F$ is convex regardless of the kernel
properties. This is a significant benefit of the algorithm which
enables to extend its use beyond what algorithms such as SVMs
require. In particular, some standard distances known not to be PSD
such as the edit-distance and many others could be used with this
algorithm.  Another advantage of this algorithm compared to standard
SVM and other $\ell_2$-regularized methods is that $\ell_1$-norm
regularization used for \AlgoName~ leads sparse solutions. The solution
$\Alpha$ is typically sparse, which significantly reduces prediction
time and the memory needs.

Note that hypotheses $h \in H_k$ are defined by
$h(x) = K_k(x, x')$ where $x'$ is an arbitrary element of the
input space $\cX$. However, our objective only includes
those $x_j$ that belong to the observed sample.
We show that in the case of a PDS kernel,
there is no loss of generality in that as we now show.\ignore{
Observe that instead of optimizing over an arbitrary $x' \in \cX$ for
each $K_k$, we only consider functions $x \mapsto K_k(x, x_j)$, where
$x_j$ is a sample point.} Indeed, observe that for $x' \in \cX$ we can
write $\Phi_k(x') = \bw + \bw^\perp$, where $\Phi_k$ is a feature map
associated with the kernel $K_k$ and where $\bw$ lies in the span of
$\Phi_k(x_1), \ldots, \Phi_k(x_m)$ and $\bw^\perp$ is in orthogonal
compliment of this subspace.  Therefore, for any sample point $x_i$
\begin{align*}
K_k(x_i, x') &= \langle \Phi(x_i), \Phi(x') \rangle_{\mathcal{H}_k}
= \langle \Phi(x_i), \bw \rangle_{\mathcal{H}_k} + 
\langle \Phi(x_i), \bw^\perp \rangle_{\mathcal{H}_k} \\
& = \sum_{j=1}^m \beta_{j} \langle \Phi(x_i), \Phi(x_j) \rangle_{\mathcal{H}_k}
=
\sum_{j=1}^m \beta_{j} K_k(x_i, x_j),
\end{align*}
which leads to objective \eqref{eq:dsvm_objective}.  Note that since
selecting $-K_k(\cdot, x_j)$ with weight $\alpha_{k,j}$ is equivalent
to selecting $K_k(\cdot, x_j)$ with $-\alpha_{k, j}$, which accounts
for the absolute value on the $\alpha_{k,j}$s in the regularization
term.

The \AlgoName~ algorithm has some connections with other algorithms
previously described in the literature. In the absence of any
regularization, that is $\lambda = 0$ and $\beta = 0$, it reduces to
the minimization of the Hinge loss and is therefore of course close to
the SVM algorithm \citep{CortesVapnik1995}. For $\lambda = 0$, that is
when discarding our regularization based on the different complexity
of the hypothesis sets, the algorithm coincides with an algorithm
originally described by \cite{Vapnik98}[pp.~426-427], later by
several other authors starting with
\citep{ZhuRossetHastieTibshirani2003}, and sometimes referred to as
the norm-1 SVM.

\section{Learning Guarantees}
\label{sec:guarantees}

In this section, we provide strong data-dependent learning guarantees
for the \AlgoName~ algorithm.

Let $\cF$ denote $\conv(\bigcup_{k = 1}^p H_k)$, that is the family of
functions $f$ of the form $f = \sum_{t = 1}^T \alpha_t h_t$, where
$\Alpha = (\alpha_1, \ldots, \alpha_T)$ is in the simplex $\Delta$ and
where, for each $t \in [1, T]$, $H_{k_t}$ denotes the hypothesis set
containing $h_t$, for some $k_t \in [1, p]$. Then, the following
learning guarantee holds for all $f \in \cF$
\citep{CortesMohriSyed2014}.

\begin{theorem}
\label{th:binary_classification}
  Assume $p > 1$. Fix $\rho > 0$. Then, for any $\delta > 0$, with
  probability at least $1 - \delta$ over the choice of a sample $S$ of
  size $m$ drawn i.i.d.\ according to $\cD^m$, the following inequality
  holds for all $f = \sum_{t = 1}^T \alpha_t h_t \in \cF$:
\begin{multline*}
R(f) \leq \h R_{S, \rho}(f) +
\frac{4}{\rho} \sum_{t = 1}^T \alpha_t \R_m(H_{k_t}) + \frac{2}{\rho}
\sqrt{\frac{\log p}{m}}
+ \sqrt{
\bigg\lceil \frac{4}{\rho^2} \log \Big[\frac{\rho^2 m}{\log p}\Big]
\bigg\rceil \frac{\log p }{m} + \frac{\log \frac{2}{\delta}}{2m}}.
\end{multline*}
Thus, $R(f) \leq \h R_{S, \rho}(f) + \frac{4}{\rho} \sum_{t = 1}^T
\alpha_t \R_m(H_{k_t}) +  O\left(\sqrt{\frac{\log
    p}{\rho^2 m} \log \big[ \frac{\rho^2 m}{\log p} \big] }\right)$.
\end{theorem}

\ignore{ 
This result is remarkable since the complexity term in the right-hand
side of the bound admits an explicit dependency on the mixture
coefficients $\alpha_t$. It is a weighted average of Rademacher
complexities with mixture weights $\alpha_t$, $t \in [1, T]$.  Thus,
the second term of the bound suggests that, while some hypothesis sets
$H_k$ used for learning could have a large Rademacher complexity, this
may not be detrimental to generalization if the corresponding total
mixture weight (sum of $\alpha_t$s corresponding to that hypothesis
set) is relatively small. Such complex families offer the potential of
achieving a better margin on the training sample.}

Theorem~\ref{th:binary_classification} can be used to derive \DeepSVM~
objective and we provide full details of this derivation in
Appendix~\ref{sec:opt_derivation}.
Furthermore,
the results of Theorem~\ref{th:binary_classification} can further be improved
using local Rademacher complexity analysis showing that faster rates
of convergence are possible.

\begin{theorem}
\label{th:local}
  Assume $p > 1$. Fix $\rho > 0$. Then, for any $\delta > 0$, with
  probability at least $1 - \delta$ over the choice of a sample $S$ of
  size $m$ drawn i.i.d.\ according to $\cD^m$, the following inequality
  holds for all $f = \sum_{t = 1}^T \alpha_t h_t \in \cF$ for any $K>1$:
\begin{align*}
R(f) - \frac{K}{K-1} \h R_{S, \rho}(f) & \leq
6K \frac{1}{\rho} \sum_{t=1}^T \alpha_t \R_m(H_{k_t})\\ &+   
 40 \frac{K}{\rho^2} \frac{\log p}{m} +  
 5K \frac{\log \tfrac{2}{\delta}}{m} 
 + 
5K \Bigg\lceil \frac{8}{\rho^2}
\log \frac{\rho^2 (1 + \frac{K}{K-1}) m}{40K \log p} \Bigg \rceil
\frac{\log p}{m}.
\end{align*}
Thus, for $K=2$, $R(f) \leq 2 \h R_{S, \rho}(f)+ 
\frac{12}{\rho} \sum_{t=1}^T \alpha_t \R_m(H_{k_t})
+ O\Bigg(\frac{\log p}{\rho^2 m}  \log \Big(\frac{\rho m}{\log p}\Big)
+ \frac{\log \tfrac{1}{\delta}}{m} \Bigg)$.
\end{theorem}

The proof of this result is given in Appendix~\ref{sec:proofs}.
Note that $O(\log m/\sqrt{m})$ in Theorem~\ref{th:binary_classification}
is replaced with $O(\log m/m)$ in Theorem~\ref{th:local}.
For full hypothesis classes $H_k$s, $\R_m(H_k)$ may be on the
order of $O(1/\sqrt{m})$ and will dominate the bound.
However, if we use localized classes
$H_k(r) = \set{h \in H_k \colon \E [h^2] < r}$ then
for certain values of $r^*$ local Rademacher complexities
$\R_m(H_k(r^*)) \in O(1/m)$ leading to even stronger learning guarantees.
Furthermore, this result leads to an extension of \AlgoName~
objective:
\begin{align}
\label{eq:local-obj}
 \! F(\Alpha) \!=\! \frac{1}{m} \!\sum_{i = 1}^m \max\bigg(0, 1
    \!-\! y_i y_j
 \sum_{j = 1}^m \sum_{k=1}^p \alpha_{k,j} K_k(x_i, x_j) \bigg) \!
+\! \sum_{j =1}^m \sum_{k=1}^p (\lambda \R_m(H_k(s))   + \beta)
  |\alpha_{k, j}|, \!
\end{align}
which is optimized over $\Alpha$ and parameter $s$ is set via
cross-validation. In Section~\ref{sec:complexities},
we provide an explicit expression for the local
Rademacher complexities of PDS kernel functions.  

\section{Optimization Solutions}
\label{sec:opt}

In this section, we propose two different algorithmic approaches to
solve the optimization problem \eqref{eq:dsvm_objective}: a linear
programming (LP) and a coordinate descent (CD) approach.

\subsection{Linear Programming (LP) formulation}

This section presents a linear programming approach for solving the
\AlgoName~ optimization problem \eqref{eq:dsvm_objective}. Observe that
by introducing slack variables $\xi_i$ the optimization can be
equivalently written as follows:
\begin{align*}
\min_{\Alpha, \xi} & \mspace{15mu} \frac{1}{m} \sum_{i=1}^m \xi_i +
                     \sum_{j=1}^m \sum_{k = 1}^p \Lambda_k
                     |\alpha_{k,j}| \mspace{15mu}  \text{s.t.}  \mspace{15mu} \xi_i \geq 1-  \sum_{j = 1}^m \sum_{k=1}^p \alpha_{k,j} y_i y_j
K_k(x_i, x_j), \forall i \in [1, m].
\end{align*}
Next, we introduce new variables $\alpha^{+}_{k,j} \geq 0$ and
$\alpha^{-}_{k, j} \geq 0$ such that
$\alpha_{k, j} = \alpha^{+}_{k, j} - \alpha^{-}_{k, j}$.  Then, for
any $k$ and $j$, $|\alpha_{k,j}|$ can be rewritten as
$|\alpha_{k,j}| \leq \alpha^{+}_{k,j} + \alpha^{-}_{k,j}$.
The
optimization problem is therefore equivalent to the following:
\begin{align*}
\min_{\Alpha^+ \geq 0, \Alpha^- \geq 0, \xi} & \mspace{15mu} \frac{1}{m} \sum_{i=1}^m \xi_i +
\sum_{j = 1}^m \sum_{k=1}^p \Lambda_k (\alpha^{+}_{k,j} + \alpha^{-}_{k,j})\\
\text{s.t.} & \mspace{15mu}
\xi_i \geq 1- \sum_{j = 1}^m \sum_{k=1}^p (\alpha^{+}_{k,j}
 - \alpha^{-}_{k,j}) y_i y_j  K_k(x_i, x_j) , \forall i \in [1, m],
\end{align*}
since conversely, a solution with
$\alpha_{k, j} = \alpha^{+}_{k, j} - \alpha^{-}_{k, j}$ verifies the
condition $\alpha^{+}_{k, j} = 0$ or $\alpha^{-}_{k, j} = 0 $ for any $k$
and $j$, thus $\alpha_{k, j} = \alpha^{+}_{k, j}$ when
$\alpha_{k, j} \geq 0$ and $\alpha_{k, j} = \alpha^{-}_{k, j}$ when
$\alpha_{k, j} \leq 0$. This is because if
$\delta = \min(\alpha^{+}_{k, j}, \alpha^{-}_{k, j}) > 0$, then
replacing $\alpha^{+}_{k, j}$ with $\alpha^{+}_{k, j} - \delta$ and
$\alpha^{-}_{k, j}$ with $\alpha^{-}_{k, j} - \delta$ would not affect
$\alpha^{+}_{k, j} - \alpha^{-}_{k, j}$ but would reduce
$\alpha^{+}_{k,j} + \alpha^{-}_{k,j}$.

Note that the resulting optimization problem is an LP problem since
the objective function is linear in both $\xi_i$s and $\Alpha^+$,
$\Alpha^-$, and since the constraints are affine.  There is a battery
of well-established methods to solve this LP problem including
interior-point methods and the simplex algorithm. An additional
advantage of this formulation of the \AlgoName~ algorithm is that there
is a large number of generic software packages for solving LPs making the
\AlgoName~ algorithm easier to implement.

\subsection{Coordinate Descent (CD) formulation}

An alternative approach for solving the \AlgoName~ optimization problem
\eqref{eq:dsvm_objective} consists of using a coordinate descent
method.  The advantage of such a formulation over the LP
formulation is that there is no need to explicitly store the whole
vector of $\Alpha$s but rather only non-zero entries.  This enables
learning with very large number of base hypotheses including scenarios
in which the number of base hypotheses is infinite. The full
description of the algorithm is given in Appendix~\ref{sec:cd}.

\subsection{Complexity penalties}
\label{sec:complexities}

An additional benefit of the learning bounds presented in
Section~\ref{sec:guarantees} is that they are data-dependent.  They
are based on the Rademacher complexity $r_k$s of the base hypothesis
sets $H_k$, which in some cases can be well estimated from the
training sample. Our formulation directly inherits this advantage.
However, in certain cases computing or estimating complexities
$r_1, \ldots, r_j$ may be costly.  In this section, we discuss various
upper bounds on these complexities that be can used in practice for
efficient implementation of the \AlgoName~ algorithm.

Note that the hypothesis set
$H_k = \set{x \mapsto \pm K_k(x, x') \colon x' \in \cX}$ is of course
distinct from the RKHS $\scrH_k$ of the kernel $K_k$. Thus, we cannot
use the known upper bound on $\h \R_S(\scrH_k)$ to bound
$\h \R_S (H_k)$. Nevertheless our proof of the upper bound is similar
and leads to a similar upper bound.  

\begin{lemma}
\label{lemma:trace}
Let $\K_k$ be the kernel matrix of the PDS kernel function $K_k$ for the
sample $S$ and let $\kappa_k = \sup_{x \in \cX} \sqrt{K_k(x,
  x)}$. Then, the following inequality holds:
\begin{equation*}
\h \R_S (H_k) \leq \frac{\kappa_k \sqrt{\Tr[\K_k]}}{m}.
\end{equation*}
\end{lemma}
\ignore{
\begin{proof}
$\h \R_S (H_k)$ can
be upper bounded as follows using the Cauchy-Schwarz inequality:
\begin{align*}
\h \R_S (H_k)
& = \frac{1}{m} \E_\ssigma \bigg[\sup_{x' \in \cX, s \in \set{-1, +1}}
  \sum_{i = 1}^m \sigma_i s K_k(x_i, x') \bigg]
= \frac{1}{m} \E_\ssigma \bigg[\sup_{x' \in \cX}
  \Big| \sum_{i = 1}^m \sigma_i s K_k(x_i, x') \Big| \bigg]\\
& = \frac{1}{m} \E_\ssigma \bigg[\sup_{x' \in \cX} \Big| \sum_{i = 1}^m
  \sigma_i \Phi_k(x_i) \cdot \Phi_k(x') \Big| \bigg] \leq \frac{1}{m} \E_\ssigma \bigg[\sup_{x' \in \cX} \| \Phi_k(x') \|_{\scrH_k}
  \Big\| \sum_{i = 1}^m
  \sigma_i \Phi_k(x_i) \Big\|_{\scrH_k} \bigg]\\
& = \frac{\kappa_k}{m} \E_\ssigma \bigg[  \Big\| \sum_{i = 1}^m
  \sigma_i \Phi_k(x_i) \Big\|_{\scrH_k} \bigg]  \leq \frac{\kappa_k}{m} \sqrt{\E_\ssigma \bigg[  \sum_{i, j = 1}^m
  \sigma_i \sigma_j \Phi_k(x_i) \cdot \Phi_k(x_j)  \bigg]}
= \frac{\kappa_k \sqrt{\Tr[\K_k]}}{m},
\end{align*}
where we used in the last line Jensen's inequality.
\end{proof}}
We present the full proof of this result in Appendix~\ref{sec:proofs}.
Observe that the expression given by the lemma can be precomputed and used as the parameter $r_k$ of the
optimization procedure.

The upper bound just derived is not fine enough to distinguish between
different normalized kernels since for any normalized kernel $K_k$,
$\kappa = 1$ and $\Tr[\K_k] = m$. In that case, finer bounds in terms
of localized complexities can be used. In particular, local Rademacher
complexity of a set of functions $H$ id defined as
$\R^\text{loc}_m(H, r) = \R_m(\set{h \in H \colon \E[h^2] \leq r})$.
If $(\lambda_i)_{i=1}^\infty$ is a sequence of eigenvalues associated
with the kernel $K_k$ then once can show
\citep{Mendelson2003,BartlettBousquetMendelson2005} that for every
$r > 0, 
\R^\text{loc}_m(H, r) \leq \sqrt{\frac{2}{m} \min_{\theta \geq 0}
\Big(\theta r + \sum_{j > \theta} \lambda_j \Big)} =
\sqrt{\frac{2}{m} \sum_{j=1}^\infty \min(r, \lambda_j)}.$

Furthermore, there is an absolute constant $c$ such that if
$\lambda_1 \geq \frac{1}{m}$, then for every $r \geq \frac{1}{m}$,
\begin{align*}
\frac{c}{\sqrt{m}} \sum_{j=1}^\infty (r, \lambda_j) \leq
\R^\text{loc}_m(H, r).
\end{align*}
Note that taking $r = \infty$ recovers earlier bound
$\R_m(H_k) \leq \sqrt{\Tr[\K_k] / m}$.
On the other hand one can show that for instance in the case
of Gaussian kernels $\R^\text{loc}_m(H, r) = O(\sqrt{\frac{r}{m} \log (1/r)})$
and using the fixed point of this function leads to
$\R^\text{loc}_m(H, r) = O(\frac{\log m}{m})$. These results can be
used in conjunction with the local Rademacher complexity extension of
\AlgoName~ discussed in Section~\ref{sec:guarantees}.

If all of the kernels belong to the same family such as,
for example, polynomial or Gaussian kernels it may be
desirable to use measures of complexity that would account
for specific properties of the given family of kernels such
polynomial degree or bandwidth of the Gaussian. Below we discuss
several additional upper bounds that aim to address these questions. 

For instance, if $K_k$ is a polynomial kernel of degree $k$, then we
can use an upper bound on the Rademacher complexity of $H_k$ in terms
of the square-root of its pseudo-dimension $\Pdim(H_k)$, which
coincides with the dimension $d_k$ of the feature space corresponding
to a polynomial kernel of degree $k$, which is given by \begin{equation}
\label{eq:dk}
d_k = \binom{N + k}{k} \leq \frac{(N + k)^k}{k!} \leq
\bigg(\frac{(N + k)e}{k}\bigg)^k.
\end{equation}

\begin{lemma}
\label{lemma:dk}
Let $K_k$ be a polynomial kernel of degree $k$. Then, the empirical
Rademacher complexity of $H_k$ can be upper bounded as $
\h \R_S(H_k) \leq 12 \kappa_k^2 \sqrt{\frac{\pi d_k}{m}}. $
\end{lemma}

\ignore{
\begin{proof}
By the proof of Lemma~\ref{lemma:trace},
we can write
\begin{equation*}
\h \R_S(H_k) \leq \frac{\kappa_k}{m} \E_\ssigma \bigg[  \Big\| \sum_{i = 1}^m
  \sigma_i \Phi_k(x_i) \Big\|_{\scrH_k} \bigg] = 2 \kappa_k^2 \, \h \R_S(H^1_k),
\end{equation*}
where $H^1_k$ is the family of linear functions 
$H^1_k = \set{\bw \mapsto \bw \cdot \Phi_k(x) \colon \| \bw \|_{\scrH_k}
  \leq \frac{1}{2 \kappa_k}}$.
By Dudley's formula \citep{Dudley1989}, we can write
\begin{equation*}
\h \R_S(H^1_k) \leq 12 \int_{0}^\infty \sqrt{\frac{\log \cN(\e, H^1_k, L_2(\h
    \cD))}{m}} d\e,
\end{equation*}
where $\h \cD$ is the empirical distribution.  Since $H^1_k$ can be
viewed as a subset of a $d_k$-dimensional linear space and since
$| \bw \cdot \Phi_k(x) | \leq \frac{1}{2}$ for all $x \in \cX$ and $w \in H^1_k$,
we have
$\log \cN(\e, H^1_k, L_2(\h \cD)) \leq \log
\big[(\frac{1}{\e})^{d_k}\big]$. Thus, we can write
\begin{align*}
\h \R_S(H^1_k) \leq 12 \int_{0}^1 \sqrt{\frac{d_k \log \frac{1}{\e}}{m}}
  d\e
= 12 \sqrt{\frac{d_k}{m}} \int_{0}^1 \sqrt{\log \frac{1}{\e}} d\e
= 12 \sqrt{\frac{d_k}{m}} \frac{\sqrt{\pi}}{2},
\end{align*}
which completes the proof.
\end{proof}} 

The proof of this result is in Appendix~\ref{sec:proofs}
Thus, in view of the lemma, we can use $r_k = \kappa_k^2 \sqrt{d_k}$ as
a complexity penalty in the formulation of the \AlgoName~ algorithm with
polynomial kernels, with $d_k$ given by the expression \eqref{eq:dk}.

\ignore{
For instance, if $K_k$ is a polynomial kernel of degree $k$
then we can use an upper bound on Rademacher complexity in terms of $\VCdim$
dimension $d_k$ given by
\begin{align*}
\R_m(H_k) \leq \sqrt{\frac{d_k \log m}{m}},
\end{align*}
where the dimension $d_k$ of the feature space corresponding to a
polynomial kernel of degree $k$ is given by
\begin{align*}
d_k = \binom{N + k}{k} \leq \frac{(N + k)^k}{k!} \leq
\bigg(\frac{(N + k)e}{k}\bigg)^k.
\end{align*}
Thus, we can use $r_k = \sqrt{d_k}$ as a complexity penalty in the
formulation of the \AlgoName~ algorithm with polynomial kernels, with $d_k$
given by the expression above.
}

\ignore{
In other cases, when $H_k$
is a family of functions taking binary values, we can use an upper
bound on the Rademacher complexity in terms of the growth function of
$H_k$, $\Pi_{H_k}(m)$: $\R_m(H_k) \leq \sqrt{\frac{2 \log
    \Pi_{H_k}(m)}{m}}$.  
}

\ignore{
From the above bound, we have the following optimization problem for kernel-based hypotheses classes:
\begin{align*}
\min_{\alpha, s} \frac{1}{m} \sum_{i=1}^{m} \Phi\left(1-y_i \sum_{t=1}^{T}{\alpha_t h_t(x_i)}\right)
+ C \sum_{t=1}^{T} \alpha_t \sqrt{\frac{2}{m} \sum_{l=1}^{\infty}\min\{s_t, \lambda_{t,l}\}}
\end{align*}}

\section{Experiments}
\label{sec:experiments}

\begin{table}[t]

\scriptsize
\centering
\begin{tabular}{|c||c|c|c|c||c|c|c|c|}
\hline
                              & \multicolumn{4}{c|}{Error (\%)}       & \multicolumn{4}{c|}{Number of support vectors} \\
\cline{2-9}
Dataset                       & L2 SVM  & L1 SVM  & \DeepSVM 2   & \DeepSVM c   & L2 SVM     & L1 SVM    & \DeepSVM 2     & \DeepSVM c     \\
\hline
                              & Mean    & Mean    & Mean    & Mean    & Mean       & Mean      & Mean      & Mean      \\
                              & (Stdev) & (Stdev) & (Stdev) & (Stdev) & (Stdev)    & (Stdev)   & (Stdev)   & (Stdev)   \\
\hline
\multirow{2}{*}{ocr49}        & 5.05    & 3.50    & {\bf 2.70}    & {\bf 3.50}    & 449.8     & 140.0    & 6.8      & 164.6    \\
                              & (0.65)    & (0.85)    & (0.97)    & (0.85)    & (3.6)       & (3.6)      & (1.3)      & (9.5)      \\
\hline
\multirow{2}{*}{phishing}     & 4.64    & 4.11    & \textit{ \textbf{ 3.62}}    & \textit{ \textbf{ 3.87}}    & 221.4     & 188.8    & 73.0     & 251.8    \\
                              & (1.38)    & (0.71)    & (0.44)    & (0.80)    & (15.1)      & (7.5)      & (3.2)      & (4.0)      \\
\hline
\multirow{2}{*}{waveform01}   & 8.38    & 8.47    & 8.41    & 8.57    & 415.6     & 13.6     & 18.4     &    14.6       \\
                              & (0.63)    & (0.52)    & (0.97)    & (0.58)    & (8.1)       & (1.3)      & (1.5)      &   (2.3)        \\
\hline
\multirow{2}{*}{breastcancer} & 11.45   & 12.60   & \textit{ \textbf{ 11.73}}   & \textit{ \textbf{ 11.30}}   & 83.8      & 46.4     & 66.6     & 29.4     \\
                              & (0.74)    & (2.88)    & (2.73)    & (1.31)    & (10.9)      & (2.4)      & (3.9)      & (1.9)      \\
\hline
\multirow{2}{*}{german}       & 23.00   & 22.40   & 24.10   & 24.20   & 357.2     & 34.4     & 25.0     & 30.2     \\
                              & (3.00)    & (2.58)    & (2.99)    & (2.61)    & (16.7)      & (2.2)      & (1.4)      & (2.3)      \\
\hline
\multirow{2}{*}{ionosphere}   & 6.54    & 7.12    & \textit{ \textbf{ 4.27} }   & \textit{ \textbf{ 3.99}}    & 152.0     & 73.8     & 43.6     & 30.6     \\
                              & (3.07)    & (3.18)    & (2.00)    & (2.12)    & (5.5)       & (4.9)      & (2.9)      & (1.8)      \\
\hline
\multirow{2}{*}{pima}         & 31.90   & 30.85   & 31.77   & \textit{ \textbf{ 30.73}}   & 330.0     & 26.4     & 33.8     & 40.6     \\
                              & (1.17)    & (1.54)    & (2.68)    & (1.46)    & (6.6)       & (0.6)      & (3.6)      & (1.1)      \\
\hline
\multirow{2}{*}{musk}         & 15.34   & 11.55   & {\bf 10.71}   & {\bf 9.03}    & 251.8     & 115.4    & 125.6    & 108.0    \\
                              & (2.23)    & (1.49)    & (1.13)    & (1.39)    & (12.4)      & (4.5)      & (8.0)      & (5.2)      \\
\hline
\multirow{2}{*}{retinopathy}  & 24.58   & 24.85   & 25.46   & {\bf 24.06}   & 648.2     & 42.6     & 43.6     &   48.0        \\
                              & (2.28)    & (2.65)    & (2.08)    & (2.43)    & (21.3)      & (3.7)      & (4.0)      &    (3.1)       \\ 
\hline
\multirow{2}{*}{climate}      & 5.19    & 5.93    & 5.56    & 6.30    & 66.0      & 19.0     & 51.0     & 18.6     \\
                              & (2.41)    & (2.83)    & (2.85)    & (2.89)    & (4.6)       & (0.0)      & (6.7)      & (0.9)      \\
\hline
\multirow{2}{*}{vertebral}    & 17.74   & 18.06   & 17.10   & 17.10   & 75.4      & 4.4      & 9.6      & 8.2      \\
                              & (6.35)    & (5.51)    & (7.27)    & (6.99)    & (4.0)       & (0.6)      & (1.1)      & (1.3)     \\
\hline
\end{tabular}
\label{table:poly}
\caption[]{Experimental results with \AlgoName~ and polynomial
  kernels. \DeepSVM c refers to the algorithm obtained by using
  Lemma~\ref{lemma:trace} as complexity measure, while \DeepSVM 2 refers to
  the algorithm obtained by using Lemma~\ref{lemma:dk}. Indicated in
  boldface are results where the errors obtained are statistically
  significant at a confidence level of 5\%. In italics are
  results that are better at 10\% level. 
}

\end{table}
 
We experimented with several benchmark datasets from the UCI 
repository, specifically {\tts breastcancer}, {\tts climate}, {\tts diabetes}, {\tts german(numeric)}, {\tts ionosphere}, {\tts musk}, {\tts ocr49}, 
{\tts phishing}, {\tts retinopathy}, {\tts vertebral} and {\tts waveform01}.
Here, {\tts ocr49} refers to the subset of the OCR dataset with classes
4 and 9, and similarly {\tts waveform01} refers to the subset of 
{\tts waveform} dataset with classes 0 and 1. More details on all the
datasets are given in Table~\ref{table:b-datasets}
in Appendix~\ref{sec:datasets}.

Our experiments compared \AlgoName~ to regular SVM, that we refer to as
$L_2$-SVM, and to norm-1 SVM, called $L_1$-SVM.  In all of our
experiments, we used {\tt lp\_solve}, an off-the-shelf LP solver, to
solve the \AlgoName~ and $L_1$-SVM optimization problems.  For
$L_2$-SVM, we used {\tt LibSVM}.

In each of the experiments, we used standard 5-fold cross-validation 
for performance evaluation and model selection.
In particular, each dataset was randomly partitioned into 5 folds, and
each algorithm was run 5 times, with a different assignment of folds
to the training set, validation set and test set for each
run. Specifically, for each $i \in \{0, \ldots, 4\}$, 
fold $i$ was used for testing, fold $i + 1~(\text{mod}~5)$ was used for
validation, and the remaining folds were used for training. For each
setting of the parameters, we computed the average validation
error across the 5 folds, and selected the parameter setting with
minimum average validation error. The average error across the
5 folds was then computed for this particular parameter setting.

In the first set of experiments we  
used polynomial kernels of the form 
$K_k({\bf x}, {\bf y}) = ({\bf x}^{T}{\bf y} + 1)^{k}$.
We report the results in Table~\ref{table:poly}. 
For \AlgoName~, we 
optimized over $\lambda \in \{10^{-i} \colon i = 0, \ldots, 6\}$ and
$\beta \in \{10^{-i} \colon i = 0, \ldots, 6\}$
The family of kernel functions
$H_k$ for $k \in [1,10]$ was chosen to be
the set of polynomial kernels of degree $k$.
In our experiments
we compared the bounds of both Lemma~\ref{lemma:trace} and Lemma~\ref{lemma:dk}
used as an estimate of the Rademacher complexity.
For $L_1$-SVM, we cross-validated
over degrees in range 1 through 10 and $\beta$ in the same
range as for \AlgoName~. Cross-validation
for $L_2$-SVM was also done over the degree and regularization
parameter $C \in \set{10^i \colon i=-4, \ldots, 7}$. 

On 5 out of 11 datasets \AlgoName~ outperformed $L_2$-SVM
and $L_1$-SVM with a considerable improvement on 3 data sets.
On the rest of the datasets, there was no statistical difference
between these algorithms.
Note that our results are also consistent with previous studies
that indicated that $L_1$-SVM and $L_2$-SVM often have comparable
performance.
Observe that solutions obtained by \AlgoName~ are often up to 10 times
sparser then those of $L_2$-SVM. In other words, \AlgoName~
has a benefit of sparse solutions and often an improved performance,
which provides strong empirical evidence in the support of our
formulation.
\ignore{
Figure~\ref{fig:one} gives an additional insight into when it is
possible for \AlgoName~ algorithm to achieve improvement in performance.
Figure~\ref{fig:one} presents the distribution of
the total mass of $\Alpha$ parameter across kernels of different
degree. Note that for {\tt ionosphere} data set \AlgoName~
algorithm selects a wide range of different kernels
with higher mass assigned to less complex kernel functions,
which leads to a significantly better performance.
On the other hand, in the case of {\tt breastcancer} dataset
all $\Alpha$s are concentrated on a single kernel, which explains 
why in this case performance of \AlgoName~ is the same as that of
$L_1$-SVM and $L_2$-SVM. 
}
In a second set of experiments we used families of
Gaussian kernels
based on distinct values of the parameter
$\gamma \in \set{10^i \colon i = -6, \ldots, 0}$.
We used the bound of Lemma~\ref{lemma:trace} as an estimate
of the Rademacher complexity.  
In our cross-validation we used the same range for $\lambda$
and $\beta$ parameters of \AlgoName~ and $L_1$-SVM algorithms.
For $L_2$-SVM we increased the range of the regularization parameter:
$C \in \set{10^i \colon i=-4, \ldots, 7}$. The results of our experiments
are comparable to the results with polynomial kernels, however,
improvements obtained by \AlgoName~ are not always as significant in this
case. The sparseness of the solutions are comperable to those observed
with polynomial kernels.

\section{Conclusion}
\label{sec:conclusion}

In this paper we presented a new support vector algorithm - \AlgoName~.
Our algorithm benefits from strong data-dependent learning
guarantees that enable learning with highly complex feature maps
and yet not overfit. We further improved these learning guarantees
using local complexity analysis leading to an extension
of \AlgoName~ algorithm. The key ingredient of our algorithm
is a new regularization term that makes use of the
Rademacher complexities of different families of kernel functions
used by the \AlgoName~ algorithm. We provide a thorough analysis
of several different alternatives that can be used for this approximation.
We also provide two practical implementations
of our algorithm based on linear programming and coordinate descent.
Finally, we  presented results of extensive experiments
that show that our algorithm always finds solutions that are
much sparse than those of the other support vector algorithms
and at the same time often outperforms other formulations.

\ignore{
\subsubsection*{Acknowledgments}
This work was partly funded by the NSF award IIS-1117591.
}

\newpage
\bibliographystyle{abbrvnat} 

{
\bibliography{dsvm}
}

\newpage
\appendix

\section{Proofs of Learning Guarantees}
\label{sec:proofs}

\begin{reptheorem}{th:local}
  Assume $p > 1$. Fix $\rho > 0$. Then, for any $\delta > 0$, with
  probability at least $1 - \delta$ over the choice of a sample $S$ of
  size $m$ drawn i.i.d.\ according to $\cD^m$, the following inequality
  holds for all $f = \sum_{t = 1}^T \alpha_t h_t \in \cF$ for any $K>1$:
\begin{align*}
R(f) - \frac{K}{K-1} \h R_{S, \rho}(f) & \leq
6K \frac{1}{\rho} \sum_{t=1}^T \alpha_t \R_m(H_{k_t})\\ &+   
 40 \frac{K}{\rho^2} \frac{\log p}{m} +  
 5K \frac{\log \tfrac{2}{\delta}}{m} 
 + 
5K \Bigg\lceil \frac{8}{\rho^2}
\log \frac{\rho^2 (1 + \frac{K}{K-1}) m}{40K \log p} \Bigg \rceil
\frac{\log p}{m}.
\end{align*}
Thus, for $K=2$, $R(f) \leq 2 \h R_{S, \rho}(f)+ 
\frac{12}{\rho} \sum_{t=1}^T \alpha_t \R_m(H_{k_t})
+ O\Bigg(\frac{\log p}{\rho^2 m}  \log \Big(\frac{\rho m}{\log p}\Big)
+ \frac{\log \tfrac{1}{\delta}}{m} \Bigg)$.
\end{reptheorem}

\begin{proof}
For a fixed $\bh = (h_1, \ldots, h_T)$, any $\Alpha \in \Delta$
defines a distribution over $\set{h_1, \ldots, h_T}$. Sampling from
$\set{h_1, \ldots, h_T}$ according to $\Alpha$ and averaging leads to
functions $g$ of the form $g = \frac{1}{n} \sum_{i = 1}^T n_t h_t$ for
some $\n = (n_1, \ldots, n_T)$, with $\sum_{t = 1}^T n_t = n$, and
$h_t \in H_{k_t}$.

For any $\N = (N_1, \ldots, N_p)$ with $| \N | = 
n$, we consider the family of functions
\begin{equation*}
G_{\cF, \N} = \bigg\{\!\frac{1}{n} \sum_{k = 1}^p \sum_{j = 1}^{N_k}
  h_{k,j} \!\mid\! \forall (k, j) \in [p] \times [N_k], h_{k,j} \!\in\!
  H_k \!\bigg\},
\end{equation*}
and the union of all such families $G_{\cF, n} = \bigcup_{| \N | = n}
G_{\cF, \N}$. Fix $\rho > 0$.
We define a class
$\Phi \circ G_{\cF, \N} = \set{\Phi_\rho(g) \colon g \in G_{\cF, \N}}$
and $\cG_r = \cG_{\Phi, \cF, \N, r} = \set{r \ell_g / \max(r, \E[\ell_g]
\colon \ell_g \in \Phi \circ G_{\cF, \N}}$
for $r$ to be chosen later.
Observe that for $v_g \in \cG_{\Phi, \cF, \N, r}$
$\Var[v_g] \leq r$. Indeed, if $r > \E[\ell_g]$ then $v_g = \ell_g$.
Otherwise, $\Var[v_g] = r^2  \Var[\ell_g] / (\E[\ell_g])^2
\leq r (\E[\ell_g^2]) / \E[\ell_g] \leq r$.

By Theorem~2.1 in \cite{BartlettBousquetMendelson2005},
for any $\delta > 0$ with probability
at least $1-\delta$, for any $0 < \beta < 1$,
\begin{align*}
V \leq
2(1+\beta)  \R_m(\cG_{\Phi, \cF, \N, r}) +
\sqrt{ \frac{2r\log\tfrac{1}{\delta}}{m} }
+ \Big(\frac{1}{3} + \frac{1}{\beta}\Big) \frac{\log\tfrac{1}{\delta}}{m},
\end{align*}
where $V = \sup_{v \in \cG_r}  \left(\E [v] - \E_n [v]\right)$
and $\beta$ is a free parameter.
Next we observe that if $\R_m(\cG_{\Phi, \cF, \N, r}) \leq
\R_m(\set{\alpha \ell_g \colon g \in \Phi \circ G_{\cF, \N}, \alpha \in [0,1]}) = \R_m(\Phi \circ G_{\cF, \N})$. Therefore, using
Talagrand's contraction lemma and convexity we have that
$\R_m(\cG_{\Phi, \cF, \N, r}) \leq
\frac{1}{\rho} \sum_{k=1}^p \frac{N_k}{n} \R_m(H_k)$. It follows that
for any $\delta > 0$ with probability at least $1-\delta$, for all $0 < \beta < 1$
\begin{align*}
V \leq
2(1+\beta) \frac{1}{\rho} \sum_{k=1}^p \frac{N_k}{n} \R_m(H_k)  +
\sqrt{ \frac{2r\log\tfrac{1}{\delta}}{m} }
+ \Big(\frac{1}{3} + \frac{1}{\beta}\Big) \frac{\log\tfrac{1}{\delta}}{m}.
\end{align*}
Since there are at most $p^n$ possible $p$-tuples $\N$ with $| \N | = n$, by
the union bound, for any $\delta > 0$, with probability at
least $1 - \delta$, 
\begin{align*}
V \leq
2(1+\beta) \frac{1}{\rho} \sum_{k=1}^p \frac{N_k}{n} \R_m(H_k) + \sqrt{\frac{r \log\frac{p^n}{\delta}}{m}} +
 \Big(\frac{1}{3} + \frac{1}{\beta}\Big) \frac{\log \tfrac{p^n}{\delta} }{m}. 
\end{align*}
Thus, with probability at least $1 - \delta$, for all functions $g =
\frac{1}{n} \sum_{i = 1}^T n_t h_t$ with $h_t \in H_{k_t}$, the following
inequality holds
\begin{align*}
V \leq
2(1+\beta) \frac{1}{\rho} \sum_{t=1}^T \frac{n_t}{n} \R_m(H_{k_t}) + \sqrt{\frac{r \log\frac{p^n}{\delta}}{m}}  +
 \Big(\frac{1}{3} + \frac{1}{\beta}\Big) \frac{\log \tfrac{p^n}{\delta} }{m}. 
\end{align*}
Taking the expectation with respect to $\Alpha$ and using $\E_\Alpha[n_t/n] =
\alpha_t$, we obtain that for any $\delta > 0$, with probability at
least $1 - \delta$, for all $\bh$, we can write
\begin{align*}
\E_\Alpha[V] \leq
2(1+\beta) \frac{1}{\rho} \sum_{t=1}^T \alpha_t \R_m(H_{k_t}) + \sqrt{\frac{r \log\frac{p^n}{\delta}}{m}}
  +
 \Big(\frac{1}{3} + \frac{1}{\beta}\Big) \frac{\log \tfrac{p^n}{\delta} }{m}. 
\end{align*}\ignore{
By sub-root property of $\psi(r) = \sum_{t=1}^T \alpha_t \psi_{k_t}$
we have that for $r \geq r^*$,
$\psi(r) \leq \sqrt{r/r^*} \psi(r^*) =
\sqrt{r r^*}$, and thus
\begin{align*}
\E_\Alpha[V] \leq
\frac{10(1+\beta)}{\rho} \sqrt{r^* r} + \sqrt{\frac{r \log\frac{p^n}{\delta}}{m}} +
 \Big(\frac{1}{3} + \frac{1}{\beta}\Big) \frac{\log \tfrac{p^n}{\delta} }{m}. 
\end{align*}}
We now show that $r$ can be chosen in such a way that
$\E_\Alpha[V] \leq r  /  K$. The right hand side of the above bound
is of the form $A\sqrt{r} + B$. Note that
solution of $r/K = C+ A\sqrt{r}$ is bounded by $K^2 A^2 + 2 KC$ and hence
by Lemma~5 in \citep{BartlettBousquetMendelson2002} the following bound holds
\begin{align*}
&\E_\Alpha[R_{\rho/2}(g) - \frac{K}{K-1} \h R_{S, \rho}(g)]
\leq
4 K  (1+\beta) \frac{1}{\rho} \sum_{t=1}^T \alpha_t \R_m(H_{k_t}) +
\Big( 2 K^2 + 2 K \Big(\frac{1}{3} + \frac{1}{\beta}\Big)\Big)
 \frac{\log\tfrac{1}{\delta} }{m}. 
\end{align*}
Set $\beta = 1/2$, then we have that
\begin{align*}
&\E_\Alpha[R_{\rho/2}(g) - \frac{K}{K-1} \h R_{S, \rho}(g)]
\leq
6 K \frac{1}{\rho} \sum_{t=1}^T \alpha_t \R_m(H_{k_t}) +
5K \frac{\log\tfrac{1}{\delta} }{m}. 
\end{align*}
Then, for any $\delta_n > 0$, with probability at
least $1 - \delta_n$,
\begin{align*}
&\E_\Alpha[R_{\rho/2}(g) - \frac{K}{K-1} \h R_{S, \rho}(g)]
\leq 6K \frac{1}{\rho} \sum_{t=1}^T \alpha_t \R_m(H_{k_t}) +
5K \frac{\log \tfrac{p^n}{\delta_n}}{m}.
\end{align*}
Choose $\delta_n = \frac{\delta}{2 p^{n - 1}}$ for some $\delta > 0$,
then for $p \geq 2$, $\sum_{n \geq 1} \delta_n = \frac{\delta}{2 (1 -
  1/p)} \leq \delta$. Thus, for any $\delta > 0$ and any $n \geq 1$, with
probability at least $1 - \delta$, the following holds for all $\bh$:
\begin{align}
\label{eq:gbound}
&\E_\Alpha[R_{\rho/2}(g) - \frac{K}{K-1} \h R_{S, \rho}(g)]
\leq 6K \frac{1}{\rho} \sum_{t=1}^T \alpha_t \R_m(H_{k_t}) +
5K  \frac{\log \tfrac{2p^{2n-1}}{\delta}}{m}.
\end{align}
Now, for any $f = \sum_{t = 1}^T \alpha_t h_t \in \cF$ and any $g =
\frac{1}{n} \sum_{i = 1}^T n_t h_t$, we can upper bound $R(f) =
\Pr_{(x, y) \sim \cD}[y f(x) \leq 0]$, the generalization error of $f$,
as follows:
\begin{align*}
R(f)  = \Pr_{(x, y) \sim \cD}[y f(x) - y g(x) + yg(x) \leq 0]
& \leq \Pr[y f(x) - y g(x) < -\rho/2] + \Pr[y g(x) \leq \rho/2] \\
& = \Pr[y f(x) - y g(x) < -\rho/2] + R_{\rho/2}(g).
\end{align*}
We can also write
\begin{align*}
\h R_{\rho}(g) 
 = \h R_{S, \rho}(g - f + f)
 \leq \h \Pr[y g(x) - y f(x) < -\rho/2] +  \h R_{S, 3\rho/2}(f).
\end{align*}
Combining these inequalities yields
\begin{align*}
& \Pr_{(x, y) \sim \cD}[y f(x) \leq 0] - \frac{K}{K-1}\h R_{S, 3\rho/2}(f)
 \leq \Pr[y f(x) - y g(x) < -\rho/2] \\
&  + \frac{K}{K-1} \h \Pr[y g(x) - y f(x) <
-\rho/2] + R_{\rho/2}(g) - \frac{K}{K-1} \h R_{S, \rho}(g).
\end{align*}
Taking the expectation with respect to $\Alpha$ yields
\begin{align*}
 R(f) - \h R_{S, 3\rho/2}(f)
& \leq \E_{x \sim \cD, \Alpha}[1_{y f(x) - y g(x) < -\rho/2}] \\
& + \frac{K}{K-1}\E_{x \sim \cD, \Alpha}[1_{y g(x) - y f(x) <
-\rho/2}] + \E_\Alpha [R_{\rho/2}(g) - \frac{K}{K-1} \h R_{S, \rho}(g)].
\end{align*}
Since $f = \E_\Alpha[g]$, by Hoeffding's inequality, for any $x$,
\begin{align*}
& \E_\Alpha[1_{y f(x) - y g(x) < -\rho/2}] \!=\! \Pr_\Alpha[y f(x) \!-\! y g(x) \!<\! -\rho/2] \leq e^{-\frac{n\rho^2}{8}}\\
& \E_\Alpha[1_{y g(x) - y f(x) < -\rho/2}] \!=\! \Pr_\Alpha[y g(x) \!-\! y f(x) \!<\! -\rho/2] \leq e^{-\frac{n\rho^2}{8}}.
\end{align*}
Thus, for any fixed $f \in \cF$, we can write
\begin{align*}
R(f) - \h R_{S, 3\rho/2}(f)
& \leq \Big(1 + \frac{K}{K-1}\Big) e^{-n\rho^2/8} +
\E_\Alpha [R_{\rho/2}(g) -\frac{K}{K-1} \h R_{S, \rho}(g)].
\end{align*}
Thus, the following inequality holds:
\begin{align*}
\sup_{f \in \cF} \Big(R(f) - \frac{K}{K-1} \h R_{S, \rho}(f) \Big)
\leq \Big(1 + \frac{K}{K-1}\Big)
e^{-n\rho^2/8} + \sup_{\bh} \E_\Alpha [R_{\rho/2}(g) - \frac{K}{K-1}
 \h  R_{S, \rho/2}(g)].
\end{align*}
Therefore, in view of \eqref{eq:gbound}, for any $\delta > 0$ and any $n \geq 1$, with probability
at least $1 - \delta$, the following holds for all $f \in \cF$:
\begin{align*}
R(f) - \frac{K}{K-1} \h R_{S, \rho}(f) \leq 
\Big(1 + \frac{K}{K-1}\Big) e^{-n\rho^2/8} + 
6K \frac{1}{\rho} \sum_{t=1}^T \alpha_t \R_m(H_{k_t}) +
5K  \frac{\log \tfrac{2p^{2n-1}}{\delta}}{m}.
\end{align*}
To conclude the proof we optimize over $n$, $f \colon n \mapsto
v_1 e^{-n u} + v_2 n$, which leads to $n = (1/u)\log( u v_2 /  v_1)$.
Therefore, we set
\begin{align*}
n = \Bigg\lceil \frac{8}{\rho^2}
\log \frac{\rho^2 (1 + \frac{K}{K-1}) m}{40 K \log p} \Bigg \rceil
\end{align*}
to obtain that the following bound
\begin{align*}
R(f) - \frac{K}{K-1} \h R_{S, \rho}(f) & \leq
6K \frac{1}{\rho} \sum_{t=1}^T \alpha_t \R_m(H_{k_t})\\ &+   
 40 \frac{K}{\rho^2} \frac{\log p}{m} +  
 5K \frac{\log \tfrac{2}{\delta}}{m} 
 + 
5K \Bigg\lceil \frac{8}{\rho^2}
\log \frac{\rho^2 (1 + \frac{K}{K-1}) m}{40K \log p} \Bigg \rceil
\frac{\log p}{m}.
\end{align*}
Thus, taking $K=2$, simply yields
\begin{align*}
R(f) \leq 2 \h R_{S, \rho}(f)+ 
\frac{12}{\rho} \sum_{t=1}^T \alpha_t \R_m(H_{k_t})
+ O\Bigg(\frac{\log p}{\rho^2 m}  \log \Big(\frac{\rho m}{\log p}\Big)
+ \frac{\log \tfrac{1}{\delta}}{m} \Bigg)
\end{align*}
and the proof is complete.

\end{proof}

\ignore{
\begin{lemma}[Lemma~ 3.8 in \cite{BartlettBousquetMendelson2005}]
\label{lm:back-to-unweighted}
\end{lemma}

\begin{proof}
\end{proof}

\begin{lemma}[Theorem~2.1 in \cite{BartlettBousquetMendelson2005}]
\label{lm:concentration}
Let $\cF$ be a class of functions that map $\mathcal{X}$ into $[a, b]$.
Assume that there is some $r > 0$ such that for every $f \in \cF$,
$\Var[f] \leq r$. Then,
for every $\delta > 0$, with probability at least $1-\delta$, 
\begin{align*}
\sup_{f \in \cF} \left(\E [f] - \E_n [f]\right) \leq
\inf_{\beta > 0} \left( 2 (1+\beta)
{\R}_m(\cF) + \sqrt{\frac{2r\log{\frac{1}{\delta}}}{m}} + (b-a) \left( \frac{1}{3}
+\frac{1}{\beta}\right) \frac{\log{\frac{1}{\delta}}}{m}\right).
\end{align*}
\end{lemma}
}

\begin{replemma}{lemma:trace}
Let $\K_k$ be the kernel matrix of the kernel function $K_k$ for the
sample $S$ and let $\kappa_k = \sup_{x \in \cX} \sqrt{K_k(x,
  x)}$. Then, the following inequality holds:
\begin{equation*}
\h \R_S (H_k) \leq \frac{\kappa_k \sqrt{\Tr[\K_k]}}{m}.
\end{equation*}
\end{replemma}

\begin{proof}
$\h \R_S (H_k)$ can
be upper bounded as follows using the Cauchy-Schwarz inequality:
\begin{align*}
\h \R_S (H_k)
& = \frac{1}{m} \E_\ssigma \bigg[\sup_{x' \in \cX, s \in \set{-1, +1}}
  \sum_{i = 1}^m \sigma_i s K_k(x_i, x') \bigg]
= \frac{1}{m} \E_\ssigma \bigg[\sup_{x' \in \cX}
  \Big| \sum_{i = 1}^m \sigma_i s K_k(x_i, x') \Big| \bigg]\\
& = \frac{1}{m} \E_\ssigma \bigg[\sup_{x' \in \cX} \Big| \sum_{i = 1}^m
  \sigma_i \Phi_k(x_i) \cdot \Phi_k(x') \Big| \bigg] \leq \frac{1}{m} \E_\ssigma \bigg[\sup_{x' \in \cX} \| \Phi_k(x') \|_{\scrH_k}
  \Big\| \sum_{i = 1}^m
  \sigma_i \Phi_k(x_i) \Big\|_{\scrH_k} \bigg]\\
& = \frac{\kappa_k}{m} \E_\ssigma \bigg[  \Big\| \sum_{i = 1}^m
  \sigma_i \Phi_k(x_i) \Big\|_{\scrH_k} \bigg]  \leq \frac{\kappa_k}{m} \sqrt{\E_\ssigma \bigg[  \sum_{i, j = 1}^m
  \sigma_i \sigma_j \Phi_k(x_i) \cdot \Phi_k(x_j)  \bigg]}
= \frac{\kappa_k \sqrt{\Tr[\K_k]}}{m},
\end{align*}
where we used in the last line Jensen's inequality.
\end{proof}

\begin{replemma}{lemma:dk}
Let $K_k$ be a polynomial kernel of degree $k$. Then, the empirical
Rademacher complexity of $H_k$ can be upper bounded as follows:
\begin{equation*}
\h \R_S(H_k) \leq 12 \kappa_k^2 \sqrt{\frac{\pi d_k}{m}}.
\end{equation*}
\end{replemma}
\begin{proof}
By the proof of Lemma~\ref{lemma:trace},
we can write
\begin{equation*}
\h \R_S(H_k) \leq \frac{\kappa_k}{m} \E_\ssigma \bigg[  \Big\| \sum_{i = 1}^m
  \sigma_i \Phi_k(x_i) \Big\|_{\scrH_k} \bigg] = 2 \kappa_k^2 \, \h \R_S(H^1_k),
\end{equation*}
where $H^1_k$ is the family of linear functions 
$H^1_k = \set{\bw \mapsto \bw \cdot \Phi_k(x) \colon \| \bw \|_{\scrH_k}
  \leq \frac{1}{2 \kappa_k}}$.
By Dudley's formula \citep{Dudley1989}, we can write
\begin{equation*}
\h \R_S(H^1_k) \leq 12 \int_{0}^\infty \sqrt{\frac{\log \cN(\e, H^1_k, L_2(\h
    \cD))}{m}} d\e,
\end{equation*}
where $\h \cD$ is the empirical distribution.  Since $H^1_k$ can be
viewed as a subset of a $d_k$-dimensional linear space and since
$| \bw \cdot \Phi_k(x) | \leq \frac{1}{2}$ for all $x \in \cX$ and $w \in H^1_k$,
we have
$\log \cN(\e, H^1_k, L_2(\h \cD)) \leq \log
\big[(\frac{1}{\e})^{d_k}\big]$. Thus, we can write
\begin{align*}
\h \R_S(H^1_k) \leq 12 \int_{0}^1 \sqrt{\frac{d_k \log \frac{1}{\e}}{m}}
  d\e
= 12 \sqrt{\frac{d_k}{m}} \int_{0}^1 \sqrt{\log \frac{1}{\e}} d\e
= 12 \sqrt{\frac{d_k}{m}} \frac{\sqrt{\pi}}{2},
\end{align*}
which completes the proof.
\end{proof}

\section{Optimization Problem}
\label{sec:opt_derivation}

This section provides the derivation for \DeepSVM~ optimization problem.
We will assume that $H_1, \ldots, H_p$ are $p$ families of functions
with increasing Rademacher complexities $\R_m(H_k)$, $k \in [1, p]$,
and, for any hypothesis $h \in \cup_{k = 1}^p H_k$, denote by $d(h)$
the index of the hypothesis set it belongs to, that is
$h \in H_{d(h)}$. The bound of Theorem~\ref{th:binary_classification}
holds uniformly for all $\rho > 0$ and functions
$f \in \conv(\bigcup_{k = 1}^p H_k)$ at the price of an additional
term that is in
$O\Big(\sqrt{\frac{\log \log \frac{2}{\rho}}{m}}\Big)$. The condition
$\sum_{t = 1}^T \alpha_t = 1$ of
Theorem~\ref{th:binary_classification} can be relaxed to
$\sum_{t = 1}^T \alpha_t \leq 1$. To see this, use for example a null
hypothesis ($h_t = 0$ for some $t$).  Since the last term of the bound
does not depend on $\Alpha$, it suggests selecting $\Alpha$ to
minimize
\begin{equation*}
G(\Alpha) = 
\frac{1}{m} \sum_{i = 1}^m 1_{y_i \sum_{t = 1}^T \alpha_t h_t(x_i)
  \leq \rho} + \frac{4}{\rho} \sum_{t = 1}^T \alpha_t r_t,
\end{equation*}
where $r_t = \R_m(H_{d(h_t)})$. Since for any $\rho > 0$, $f$ and
$f/\rho$ admit the same generalization error, we can instead search
for $\alpha \geq 0$ with $\sum_{t = 1}^T \alpha_t \leq 1/\rho$ which
leads to
\begin{align*}
\min_{\Alpha \geq 0} \frac{1}{m} \!\sum_{i = 1}^m \!1_{y_i \!\!\sum_{t =
    1}^T \!\alpha_t h_t(x_i) \leq 1} \!+\! 4 \!\sum_{t = 1}^T \alpha_t r_t 
\, \mspace{15mu}  \text{ s.t.}  \mspace{15mu}  \sum_{t =  1}^T \alpha_t \leq \frac{1}{\rho}.
\end{align*}
The first term of the objective is not a convex function of $\Alpha$
and its minimization is known to be computationally hard. Thus, we
will consider instead a convex upper bound based on the Hinge loss:
let $\Phi(-u) = \max(0, 1 - u)$, then $1_{-u} \leq \Phi(-u)$.  Using
this upper bound yields the following convex optimization problem:
\begin{align}
\label{eq:opt}
\min_{\Alpha \geq 0} & \ \frac{1}{m} \sum_{i = 1}^m \Phi\Big(1 - y_i \sum_{t =
    1}^T \alpha_t h_t(x_i) \Big) + \lambda \sum_{t = 1}^T \alpha_t r_t
\qquad \text{s.t.} \mspace{15mu} \sum_{t =  1}^T \alpha_t \leq \frac{1}{\rho},
\end{align}
where we introduced a parameter $\lambda \geq 0$ controlling the
balance between the magnitude of the values taken by function $\Phi$
and the second term. \ignore{Note that this is a standard practice in the
field of optimization. The optimization problem in \eqref{eq:zero_one_opt} is
equivalent to a vector optimization problem, where
$(\sum_{i = 1}^m 1_{\rho_f(x_i, y_i)  \leq 1}, \sum_{t = 1}^T \alpha_t r_t)$
is minimized over $\Alpha$.
The latter problem can be scalarized leading to the introduction of a parameter $\lambda$ in \eqref{eq:opt}.} Introducing a Lagrange variable $\beta \geq 0$
associated to the constraint in \eqref{eq:opt}, the problem can be
equivalently written as
\begin{align*}
  \min_{\Alpha \geq 0} & \ \frac{1}{m} \sum_{i = 1}^m \Phi\Big(1
    - y_i \sum_{t = 1}^T \alpha_t h_t(x_i) \Big) + \sum_{t =
    1}^T (\lambda r_t + \beta ) \alpha_t.
\end{align*}
Here, $\beta$ is a parameter that can be freely selected by the
algorithm since any choice of its value is equivalent to a choice of
$\rho$ in \eqref{eq:opt}. Let $(h_{k, j})_{k, j}$ be the set of
distinct base functions $x \mapsto K_k(\cdot, x_j)$. Then, the
problem can be rewritten as $F$ be
the objective function based on that collection:
\begin{align}
\label{eq:deepboost_objective}
 \min_{\Alpha \geq 0} &  \frac{1}{m} \!\sum_{i = 1}^m \!\Phi\Big(1
    \!-\! y_i \sum_{j = 1}^N \alpha_j h_j(x_i) \Big) \!+\! \sum_{t =
    1}^N \Lambda_j \alpha_j,\!
\end{align}
with $\Alpha = (\alpha_1, \ldots, \alpha_N) \in \Rset^N$ and
$\Lambda_j = \lambda r_j + \beta$, for all $j \in [1, N]$.  This
coincides precisely with the optimization problem
$\min_{\Alpha \geq 0} F(\Alpha)$ defining \AlgoName~. Since the
problem was derived by minimizing a Hinge loss upper bound on the
generalization bound, this shows that the solution returned by \AlgoName~ benefits from the strong data-dependent learning guarantees of
Theorem~\ref{th:binary_classification}.

\section{Coordinate Descent (CD) Formulation}
\label{sec:cd}
An alternative approach for solving the \AlgoName~ optimization problem
\eqref{eq:dsvm_objective} consists of using a coordinate descent
method. A coordinate descent method proceeds in rounds. At each round,
it maintains a parameter vector $\Alpha$.  Let
$\Alpha_t = (\alpha_{t, k, j})_{k, j}^\top$ denote the vector obtained
after $t \geq 1$ iterations and let $\Alpha_0 = \0$. Let $\be_{k,j}$
denote the unit vector in direction $(k, j)$ in $\Rset^{p \times m}$ .
Then, the direction $\be_{k,j}$ and the step $\eta$ selected at the
$t$th round are those minimizing $F(\Alpha_{t - 1} + \eta \be_{k,j})$,
that is
\begin{align*}
 \! F(\Alpha) \!=\! \frac{1}{m} \!\sum_{i = 1}^m \max\Bigg(0, 1
    \!-\! y_i f_{t - 1} -  y_i y_j
 \eta K_k(x_i, x_j) \Bigg) \!
+\! \sum_{j =1}^m \sum_{k=1}^p \Lambda_k |\alpha_{t-1, j,k}|
+ \Lambda_k |\eta + \alpha_{t-1, k, j}|  ,\!
\end{align*}
where
$f_{t - 1} = \sum_{j = 1}^m \sum_{k=1}^p \alpha_{t - 1, j, k} y_j
K_k(\cdot, x_j)$.
To find the best descent direction, a coordinate descent method
computes the sub-gradient in the direction $(k, j)$ for each
$(k, j) \in [1, p] \times [1,m]$. The sub-gradient is given by
\begin{align*}
\delta F(\alpha_{t-1}, \be_j) =
\begin{cases}
\frac{1}{m} \sum_{i=1}^m 
 \phi_{t, j, k, i} + \sgn(\alpha_{t-1,k,j}) \Lambda_k
 & \text{ if } \alpha_{t-1,k,j} \not= 0 \\
0 & \text{ else if} \Big| \frac{1}{m} \sum_{i=1}^m 
 \phi_{t, j, k, i} \Big| \leq \Lambda_k \\
\frac{1}{m} \sum_{i=1}^m \phi_{t, j, k, i} -
\sgn\Big(\frac{1}{m} \sum_{i=1}^m \phi_{t, j, k, i}\Big) \Lambda_k
& \text{ otherwise }. 
\end{cases}
\end{align*}
where $\phi_{t, j, k, i} = - y_i K_k(x_i, x_j)$ if
$\sum_{k=1}^p \sum_{j=1}^m \alpha_{t-1, k, j} y_i y_j K(x_i, x_j) < 1$
and 0 otherwise.  Once the optimal direction $\be_{k,j}$ is
determined, the step size $\eta_t$ can be found using a line search or
other numerical methods.\ignore{ Note that the choice of $\eta_t$
  need to preserve the non-negativity of $\Alpha$. In other words,
  $\eta_t = \argmin_{\Alpha_{t - 1} + \eta \be_k \geq 0} F(\Alpha_{t -
    1} + \eta \be_k)$.}

The advantage of the coordinate descent formulation over the LP
formulation is that there is no need to explicitly store the whole
vector of $\Alpha$s but rather only no-zero entries.  This enables
learning with very large number of base hypotheses including scenarios
in which the number of base hypotheses is infinite.

\section{Dataset Statistics}
\label{sec:datasets}
The dataset statistics are provided in Table~\ref{table:b-datasets}

\begin{table}[t] 
\centering
\scriptsize
\caption{Dataset statistics.}
\label{table:b-datasets}
\vskip .2in
\begin{tabular}{|l|c|c|}
\hline
Data set &   Examples & Features \\ \hline
{\tts breastcancer} & 699 & 9  \\
{\tts climate} & 540 & 18  \\
{\tts diabetes} & 768 & 8 \\ 
{\tts german} & 1000 & 24 \\
{\tts ionosphere} & 351 & 34 \\ 
{\tts musk} & 476 & 166 \\ 
{\tts ocr49} & 2000 & 196 \\ 
{\tts phishing} & 2456 & 30 \\ 
{\tts retinopathy} & 1151 & 19 \\ 
{\tts vertebral} & 310 & 6 \\ 
{\tts waveform01} & 3304 & 21 \\ 
\hline
\end{tabular} 
\end{table}

\end{document}